\documentclass[10pt,twocolumn,letterpaper]{article}

\usepackage{wacv}              

\usepackage{graphicx}
\usepackage{balance}
\usepackage{amsmath,amssymb,amsfonts}
\usepackage{textcomp}
\usepackage{booktabs}
\usepackage{multirow}
\usepackage{todonotes}
\usepackage{tabularx}
\usepackage{graphicx}
\usepackage{enumerate}   
\usepackage[export]{adjustbox}
\usepackage{algorithmic}
\usepackage{color,colortbl}
\usepackage{sidecap}
\usepackage{rotating}
\usepackage{bbm}
\usepackage{float}
\usepackage{stfloats}
\usepackage{csvsimple}
\usepackage{color, colortbl}
\usepackage[english]{babel}
\usepackage{amsthm}

\newtheorem{proposition}{Proposition}

\definecolor{LightOrange}{rgb}{1, 0.96, 0.87}
\definecolor{LightCyan}{rgb}{0.88,1,1}
\newcommand{\croppdf}[1]{\IfFileExists{#1-crop.pdf}{}{\immediate\write18{pdfcrop #1.pdf}}}
\usepackage[pagebackref,breaklinks,colorlinks]{hyperref}

\newcommand{\bx}{\mathbf{x}}

\newcommand{\by}{\mathbf{y}}

\newcommand{\ba}{\mathbf{a}}
\newcommand{\bb}{\mathbf{b}}

\newcommand{\bty}{\Tilde{\mathbf{y}}}

\newcommand{\subbest}[1]{\textbf{{#1}}}

\usepackage[capitalize]{cleveref}
\crefname{section}{Sec.}{Secs.}
\Crefname{section}{Section}{Sections}
\Crefname{table}{Table}{Tables}
\crefname{table}{Tab.}{Tabs.}

\def\wacvPaperID{1238} 
\def\confName{WACV}
\def\confYear{2025}

\begin{document}

\title{Image-level Regression for Uncertainty-aware Retinal Image Segmentation}
%


%

\author{Trung DQ. Dang\thanks{Equal contributions} \and
Huy Hoang Nguyen$^{\star}$ \and
Aleksei Tiulpin \\
\and University of Oulu, Finland \\ 
{\tt\small\{trung.ng,huy.nguyen,aleksei.tiulpin\}@oulu.fi}}


%
\maketitle              

\begin{abstract}
Accurate retinal vessel (RV) segmentation is a crucial step in the quantitative assessment of retinal vasculature, which is needed for the early detection of retinal diseases and other conditions. Numerous studies have been conducted to tackle the problem of segmenting vessels automatically using a pixel-wise classification approach. The common practice of creating ground truth labels is to categorize pixels as foreground and background. This approach is, however, biased, and it ignores the uncertainty of a human annotator when it comes to annotating e.g. thin vessels. In this work, we propose a simple and effective method that casts the RV segmentation task as an image-level regression. For this purpose, we first introduce a novel Segmentation Annotation Uncertainty-Aware (SAUNA) transform, which adds pixel uncertainty to the ground truth using the pixel's closeness to the annotation boundary and vessel thickness. To train our model with soft labels, we generalize the earlier proposed Jaccard metric loss to arbitrary hypercubes for soft Jaccard index (Intersection-over-Union) optimization. Additionally, we employ a stable version of the Focal-L1 loss for pixel-wise regression. We conduct thorough experiments and compare our method to a diverse set of baselines across 5 retinal image datasets. Our empirical results indicate that the integration of the SAUNA transform and these segmentation losses led to significant performance boosts for different segmentation models. Particularly, our methodology enables UNet-like architectures to substantially outperform computational-intensive baselines (see~\cref{fig:hr_lr_comparison}). Our implementation is available at \url{https://github.com/Oulu-IMEDS/SAUNA}.
\end{abstract}

\section{Introduction}
\begin{figure}[t]
    \centering
    \croppdf{figures/SAUNAR_time_perf}
    \includegraphics[width=0.47\textwidth]{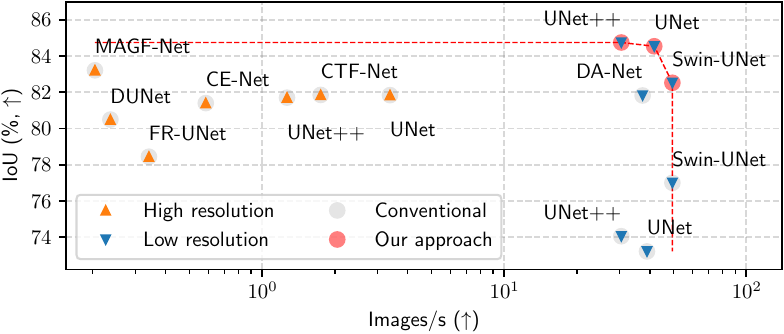}
    \caption{Comparisons of performance and throughput across methods using high-resolution (HR) and low-resolution (LR) inputs on the FIVES test set. The x-axis is in the log scale. The red line indicates the best result among the baselines. ``Conventional'' indicates baselines using binary masks (hard labels).
    The corresponding quantitative results are in~\cref{tab:exp_arch_comparisons}.
    }
    \label{fig:hr_lr_comparison}
\end{figure}

The retina serves as a non-invasive diagnostic window, providing insights into diverse clinical conditions. Quantitative assessment of retinal vasculature is essential not only for the diagnosis and prognosis of retinal diseases, but also for identifying systemic conditions such as hypertension, diabetes, and cardiovascular diseases. Numerous studies have been conducted to automate the segmentation of retinal blood vessels~\cite{li2020iternet,li2023magf,fraz2012ensemble,wang2019dual,wang2020ctf}. Typically, the problem of retinal vessel (RV) segmentation is formulated as semantic segmentation, which can be solved using Deep Learning (DL) approaches. 
\begin{figure*}[t]
    \centering
    \croppdf{figures/SAUNAR_workflow}
    \includegraphics[width=0.8\textwidth]{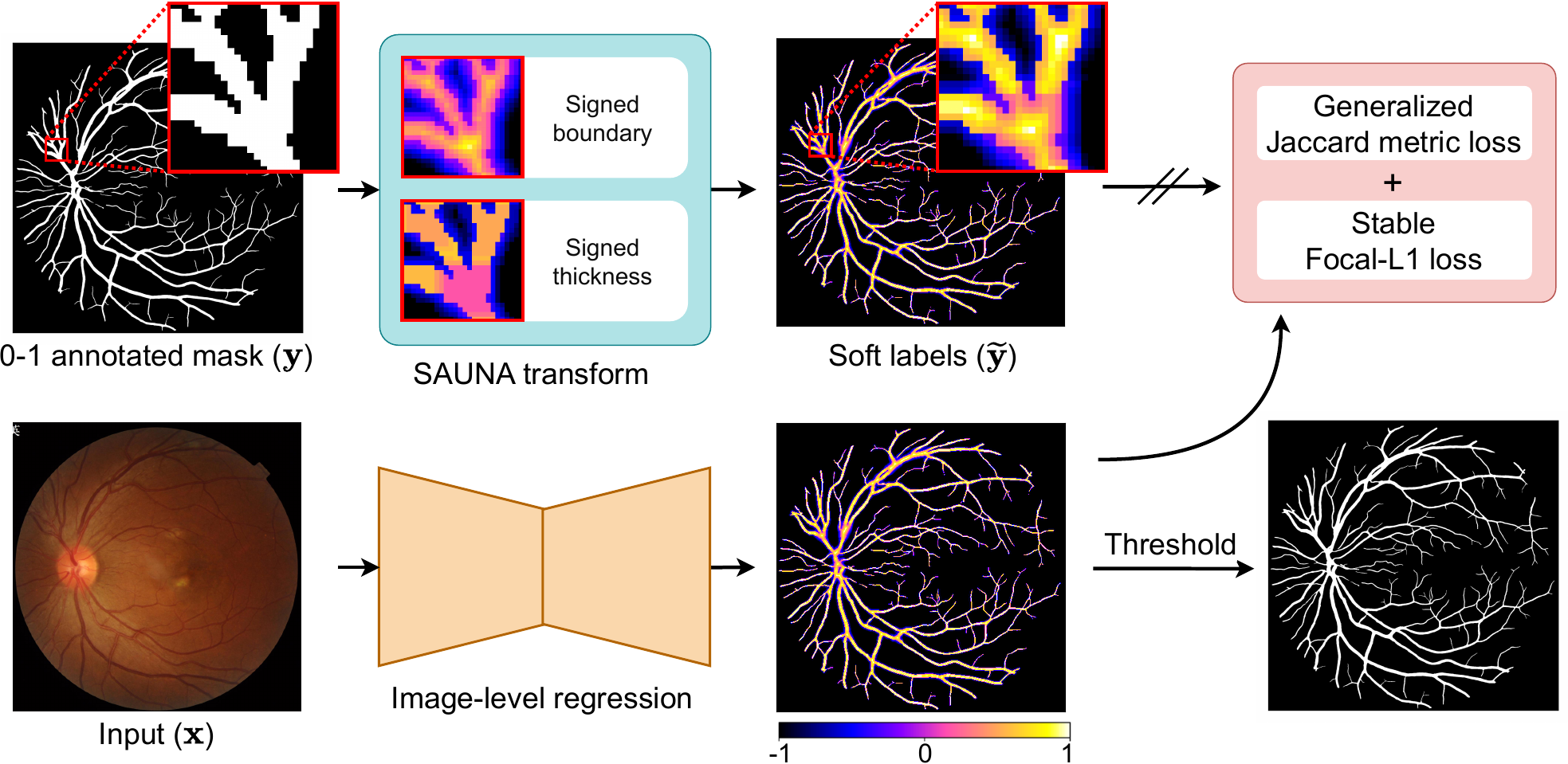}
    \caption{Our workflow of image-level regression for retinal image segmentation. Our primary contributions are the SAUNA transform (see~\cref{sc:sauna}), an extension of the Jaccard metric loss~\cite{wang2023jaccard} (see~\cref{sc:gjml}), and a stable version of the Focal-L1 loss~\cite{dang2024singr} (see~\cref{sc:focal_l1}).
    }
    \label{fig:workflow}
\end{figure*}
In semantic segmentation, the common training setup requires a collection of pairs of images and their corresponding segmentation masks (ground truth; GT), represented as $\mathbf{X} \times \mathbf{Y}$. To label a GT mask $\by \in \mathbf{Y}$, annotators  \emph{categorize} each pixel into foreground or background classes, thus termed a ``hard label''\footnote{Hereinafter, the terms ``hard label'', binary GT mask, and 0-1 GT mask are exchangeable.}.  Using such GTs, DL models for semantic segmentation are usually trained as pixel-wise empirical risk minimization problems over a dataset~\cite{bertels2019optimizing}. Therefore, the majority of prior studies primarily rely on \emph{classification losses}, such as cross-entropy loss and focal loss -- for the task as follows: $\mathcal{L}_{\mathrm{PixelsCls}} = \frac{1}{D} \sum_{l=1}^{D} \ell_{\mathrm{CLS}}(f_\theta(\bx)_{i}, \by_{i})$, where $D$ is the number of pixels, $f_\theta$ is a parametric segmentation model that takes an image $\bx$ and predicts its respective segmentation mask. The term $\ell_{\mathrm{CLS}}(\cdot, \cdot)$ is typically the cross-entropy or focal loss. Moreover, another line of research focuses on optimizing the semantic segmentation task at the image level. As such, various segmentation losses -- like Dice, Jaccard, Tversky, and Lovasz losses -- have been developed to directly address the Dice score and Intersection-over-Union~\cite{salehi2017tversky,sudre2017generalised,berman2018lovasz,bertels2019optimizing}. Despite these advancements, those losses are originally designed for hard labels.


We argue that the aforementioned categorical pixel-wise labeling approach (0 or 1 in the case of vessel segmentation) implicitly overlooks the inherent annotation process uncertainty. In retinal images (RIs), vessels are typically not discernible due to factors such as blurry boundaries, thickness, and imaging quality. Many attempts have been made to consider this matter by \textit{softening} human annotations, which discourages DL models from relying on the ``fully certain'' ground truth masks. Most prior studies tackle the problem using label smoothing techniques~\cite{silva2021using,ma2023enhanced,alcover2023soft,wang2023dice,wang2023jaccard}. The idea of label smoothing is to provide the model information that one should not assign zero probability to BG pixels when annotating,
and it was first introduced in image classification~\cite{szegedy2016rethinking}. Another line of work requires image-wise multi-annotations to model the uncertainty, which is highly expensive to obtain~\cite{silva2021using}. Other studies~\cite{liu2022combining,xue2020shape} incorporate signed distance map regression as an auxiliary task alongside segmentation losses. Still, those studies commonly approach the medical segmentation problem as a pixel classification task.

In the retinal imaging domain, most of the prior studies also follow the trend of using hard labels. Similar to DL in general, most studies in this domain focus on improving DL architectures by advancing their capacity or embedding domain knowledge into architecture design~\cite{wang2022net,qiu2023rethinking,li2020iternet,wang2019dual,li2023magf,liu2022full,gu2019net,wang2020ctf}.
However, recent studies tend to look for an optimal trade-off between the model throughput and performance~\cite{li2020iternet,wang2019dual,li2023magf,liu2022full,gu2019net,wang2020ctf}. Instead of standardizing input images into a reasonable size, those studies perform sliding windows over high-resolution RIs, which is expensive for both training and evaluation. In our work, we characterize these techniques as patch-based methods and generally question such an approach.

This work has several contributions (\cref{fig:workflow}). Firstly, we propose a new and \textit{simple} method that tackles the segmentation problem as \emph{image-level regression}. 
To achieve this, we propose a Segmentation Annotation UNcertainty-Aware (SAUNA) transform, inspired by the observation that it is hard to draw exact vessel boundaries, especially for thin vessels. The SAUNA transform generates a \emph{signed soft segmentation map} $\Tilde{\by} \in [-1,1]^{D}$ from $0-1$ annotated masks $\by$, without the need for multiple annotations per image. Specifically, positive and negative regions in $\Tilde{\by}$ represent foreground (FG) and background (BG) respectively. As BG pixels distant from the FG's vicinity are highly certain, the SAUNA transform explicitly marks them with the value $-1$. To train our model with these labels, we employ both image-level and pixel-level regression losses. For the former, we utilize the Jaccard metric loss (JML)~\cite{wang2023jaccard}. We prove that this loss can be used beyond $[0, 1]^D$ domain.
For the latter, we propose a stable version of the Focal-L1 loss~\cite{dang2024singr}. Finally, we conducted standardized and extensive experiments on $5$ RI datasets. 
Our findings indicate that while using high-resolution inputs can be beneficial, it is indeed possible to achieve both high performance and efficiency simultaneously, as illustrated in~\cref{fig:hr_lr_comparison}.

\section{Related work}

Traditional approaches to vessel segmentation in retinal images (RIs) encompass a variety of image processing and analysis techniques. The primary methods include line and edge extraction~\cite{bankhead2012fast,ricci2007retinal}, template matching~\cite{chaudhuri1989detection}, morphology-based techniques~\cite{fraz2012approach}, and probabilistic approaches~\cite{orlando2016discriminatively}. In recent years, deep neural networks have become prevalent in the semantic segmentation of generic medical images, with UNet~\cite{ronneberger2015u} popularizing the encoder-decoder architecture with skip connections. This architecture employs an encoder to extract local and global features, while the decoder merges high-level and low-level features via skip connections to predict fine-grained segmentation masks. Afterwards, UNet++ introduced by Zhou~\etal~\cite{zhou2019unet++} combines multi-scale feature maps in skip connections, further advancing segmentation performance.

In the realm of RV segmentation, numerous UNet variants have emerged to enhance feature extraction and context integration. CE-Net~\cite{gu2019net} employs dilated convolution for expanded receptive fields, while SA-UNet~\cite{guo2021sa} integrates spatial attention to capture long-range dependencies. Wang et al.\cite{wang2021retinal} introduce a Context Guided Attention Module with hard sample mining. DUNet\cite{wang2019dual} utilizes dual encoders, and Transformer modules are integrated into models~\cite{cao2022swin, lin2023stimulus}. FR-UNet~\cite{liu2022full} introduces multi-resolution convolution and feature aggregation. IterNet~\cite{li2020iternet} extends the architecture iteratively, and Li et al.~\cite{li2023magf} propose multiscale feature modules. 

Previous studies in RV segmentation often follow two main tendencies. First, to preserve the details of RVs, many approaches crop patches from high-resolution RIs for training and prediction, which significantly reduces throughput. Some studies, like CTF-Net~\cite{wang2020ctf} and DA-Net~\cite{wang2022net}, use both whole-image and patch-based information via a dual-branch approach. In contrast, we aim to develop our method in a computationally efficient setting, wherein all RIs are resized to a standard size, and processed holistically. Second, these studies typically treat RV segmentation as a pixel-classification task and focus on embedding domain knowledge into DL architectures.
In this study, we focus on incorporating uncertainty into segmentation masks. As a result, we propose a novel method to transform binary RV masks into soft labels for effective image-level regression.

Soft labels have demonstrated their potential in various domains. For instance, Xue~\etal~\cite{xue2020shape} employ signed distance maps for hippocampus segmentation, while Vasudeva~\etal~\cite{vasudeva2023geols} use the unsigned geodesic distance transform for brain magnetic resonance (MR) imaging and computed tomography (CT) scans. Additionally, Dang~\etal~\cite{dang2024singr} introduce the signed normalized geodesic transform to model uncertainty around brain tumor boundaries. However, most previous studies use soft labels for supplementary tasks. Inspired by Dang~\etal~\cite{dang2024singr}, we formulate RV segmentation as an image-level regression problem, where soft labels are our primary targets.

\section{Methodology}



\begin{figure*}[t!]
    \centering
    \croppdf{figures/SAUNA_viz}
    \croppdf{figures/input_gt}    
    \hspace*{\fill}
    \begin{subfigure}{0.12\linewidth}
        \includegraphics[width=\textwidth, valign=t]{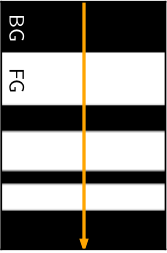}
        \caption{GT}
        \label{fig:sauna_gt}
    \end{subfigure}
    \hfill
    \begin{subfigure}{0.8\linewidth}
        \includegraphics[width=\textwidth, valign=t]{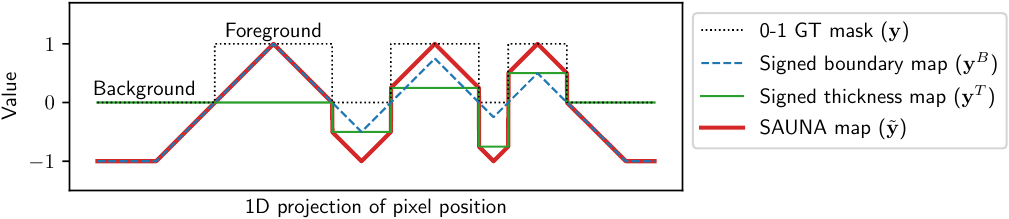}
        \caption{Transition from the GT in (a) to a SAUNA map}
        \label{fig:sauna_map}
    \end{subfigure}
    \hspace*{\fill}
    \caption{Illustration of the transformation from a 0-1 ground truth (GT) mask to its associated SAUNA map (best viewed in color): (a) 2D GT mask with an orange projection, (b) corresponding transformations in the 1D projection.}
    \label{fig:viz_sauna}
\end{figure*}

\subsection{Overview}

Due to the complexity of objects of interest in medical images, a single binary mask with $1$s and $0$s, indicating FG and BG pixels, does not properly reflect the uncertainty of the annotating process done by human annotators. Particularly, in RIs, labeling veins and arteries is highly challenging due to the imaging quality, the limited visibility of tiny branches as well as the variation of personal skills. 

To tackle the aforementioned issue, we propose a simple technique, called the Segmentation Annotating UNcertainty-Aware (SAUNA) transform (see~\cref{sc:sauna}), that takes into account the uncertainty of pixels with respect to their distances to the boundary. Following~\cite{dang2024singr}, we design SAUNA to primarily focus on the vicinity of objects of interest rather than the whole image. The SAUNA transform allows us to convert 0-1 annotated segmentation masks into heatmaps with values ranging in $[-1, 1]$.

Given the soft labels produced by the SAUNA transform, our objective is to minimize both the ``soft'' Jaccard index~\cite{wang2023jaccard} and pixel-wise similarity. To achieve this, we employ a combination of image-level and pixel-level regression losses. For image-level regression, we extend JML~\cite{wang2023jaccard}, which was originally designed to optimize ``soft'' Intersection-over-Union on the unit hypercube, to operate on an arbitrary hypercube domain, including $[-1,1]^D$ (see~\cref{sc:gjml}). For pixel-level regression, we utilize a stable version of Focal-L1 (see~\cref{sc:focal_l1}).

\subsection{Segmentation Annotating UNcertainty-Aware (SAUNA) Transform}
\label{sc:sauna}
Let $\Omega = \{1, \dots , H\} \times \{1, \dots, W\}$ denote the set of pixel coordinates. For 0-1 annotated mask $\by$, we have $\by_{i} \in \{0, 1\}, \ \forall i \in \Omega$. We firstly define the unsigned shortest Euclidean distance and thickness transforms as follows
\begin{align}
     \mathbf{d}_i &= \min_{j\in\Omega:\by_{j} \neq \by_{i}} \| i - j \|_2 , \quad i \in \Omega \\
    \mathbf{t}_i &= \max_{j \in C: \by_{i}=1,\left \| i-j \right \|_{\infty} \leq m} \mathbf{d}_{j},\quad i \in \Omega
\end{align}
where $C = \{ j \in \Omega \mid \mathbf{d}_j \leq m \}$ is the set of pixels relatively close to FG regions, and $m=\max_{k\in \Omega, \by_{k}=1} \mathbf{d}_{k}$ represents the maximum Euclidean distance from FG pixels to their nearest boundary pixels. The thickness transform is defined as the application of max-pooling to a positive distance map. As the window size $m$ is significantly large by definition, locally maximal distance values are propagated across the region. Here, we respectively introduce the signed normalized boundary and thickness transforms as follows
\begin{align}
    \by^B_{i} &= s(\by_{i}) \cdot \min \left (1, \frac{\mathbf{d}_i}{m} \right ), \quad i \in \Omega \\
    \by^T_{i} &= s(\by_{i}) \cdot \left [1 - \min \left (1, \frac{\mathbf{t}_i}{m} \right ) \right ], \quad i \in \Omega
\label{eq:boundary_uncertainty}
\end{align}
where $s(\by_{i})=\operatorname{sign}(2\by_{i} - 1)$. The minimum function is to ensure that we merely consider the neighboring regions of the boundaries, which implies ignoring distant BG pixels. $\by^B_{i} = 0$ iff $i$ corresponds to a boundary pixel. 

To this end, we propose the SAUNA transform based on boundary and thickness as the following
\begin{align}
    \bty_{i} = \by^B_{i} + \by^T_{i} \in [-1,1], \quad i \in \Omega
\end{align}
In \cref{fig:viz_sauna}, we provide a graphical illustration of how the SAUNA transform generates the signed soft labels from the 0-1 GT mask shown in~\cref{fig:sauna_gt}. In~\cref{fig:sauna_map}, $\by^B$ generates a piece-wise function that exhibits irregular zig-zag behavior when the pixel is close enough to FG regions. For distant pixels that are highly certain, it becomes a constant function with a value of ``-1''. $\by^T$ produces a step function whose value is inversely proportional to the thickness of either FG or BG region. The SAUNA map is derived from the summation of the two maps. This process preserves the behavior of $\by^B$ around the (easiest) thickest region while generating adaptive margins across the boundaries of (hard) thin ones. Intuitively, such a map encourages the image-level regression model to prioritize attention to challenging areas.                                                    

\subsection{Generalized Jaccard Metric Loss}
\label{sc:gjml}
 JML introduced by Wang~\etal~\cite{wang2023jaccard} was originally designed for soft segmentation labels, and it was used for knowledge distillation~\cite{hinton2015distilling}. The loss is formulated as
 \begin{align}\label{eq:jml}
     \Delta_{\mathrm{JML}}(\mathbf{a},\mathbf{b})= 1 - \frac{\left \| \mathbf{a} + \mathbf{b}\right \|_1 - \left \| \mathbf{a} - \mathbf{b}\right \|_1}{ \left \| \mathbf{a} + \mathbf{b}\right \|_1 + \left \| \mathbf{a} - \mathbf{b}\right \|_1},
 \end{align}
 and is proven to be a metric for any $\mathbf{a}, \mathbf{b} \in [0,1]^{D}$~\cite{wang2023jaccard}. The fact that the $\Delta_\mathrm{JML}$ loss is semi-metric or metric implies that $\forall \mathbf{a}, \mathbf{b} \in  [0,1]^{D}, \Delta_\mathrm{JML}(\mathbf{a}, \mathbf{b}) = 0 \iff \mathbf{a} \equiv \mathbf{b}$, making it an objective that directly optimizes the IoU between the predictions and the soft labels.
 To perform the \emph{image-level regression} task on the generated SAUNA maps in $[-1,1]^{D}$, we prove that the domain of $\Delta_{\mathrm{JML}}$ can be an arbitrary hypercube from $\mathbb{R}^D$, including $[-1, 1]^{D}$. 
 
 \begin{proposition}[Jaccard Metric Loss on a hypercube in $\mathbb{R}^D$]
     \label{prop:generalized_JML_metric}     
     $\Delta_{\mathrm{JML}}$ is a semi-metric in $[\alpha, \beta]^{D} \subseteq \mathbb{R}^D$.
Specifically, $\forall \ba, \bb \in [\alpha, \beta]^{D}$, we have
\begin{enumerate}[(i)]
    \item Reflexivity: $\Delta_{\mathrm{JML}}(\ba, \bb) = 0 \Longleftrightarrow \ba \equiv \bb$
    \item Positivity: $\Delta_{\mathrm{JML}}(\ba, \bb) \geq 0$
    \item Symmetry: $\Delta_{\mathrm{JML}}(\ba, \bb) = \Delta_{\mathrm{JML}}(\bb, \ba)$
\end{enumerate}
\begin{proof}
In Supplementary Sec. \textcolor{red}{1}.
\end{proof}
\end{proposition}

Given any pair of an input image and its 0-1 annotated mask $(\bx, \by)$, we utilize a parametric function $f_\theta$ to produce $\mathbf{f}=\tanh{f_\theta(\bx)}$, and apply the SAUNA transform on $\by$ to generate the soft label $\bty$. Here, both $\mathbf{f}$ and $\bty$ are in $[-1,1]^{D}$. To this end, $\forall \mathbf{f}, \bty \in [-1, 1]^D$, we rely on~\cref{prop:generalized_JML_metric} to introduce generalized JML (GJML) as follows
\begin{align}
    \mathcal{L}_{\mathrm{GJML}}(\mathbf{f}, \bty) &= \Delta_{\mathrm{JML}}(\mathbf{f}, \bty) \nonumber \\
    &= 1 - \frac{\left \| \mathbf{f} + \bty \right \|_1 - \left \| \mathbf{f} - \bty \right \|_1}{ \left \| \mathbf{f} + \bty \right \|_1 + \left \| \mathbf{f} - \bty \right \|_1}.  \label{eq:loss_gjml}
\end{align}
From~\cref{prop:generalized_JML_metric}, we have that $\mathcal{L}_{\mathrm{GJML}}$ is a semi-metric, which allows us to perform direct image-level regression for IoU maximization with soft labels.

\begin{figure}[t!]
    \centering
    \croppdf{figures/focall1_v2_loss}
    \croppdf{figures/focall1_loss}    
    \croppdf{figures/pitfall_focal_l1}    
    \subfloat[Focal-L1 loss~\cite{dang2024singr}]{
    \includegraphics[width=0.23\textwidth]{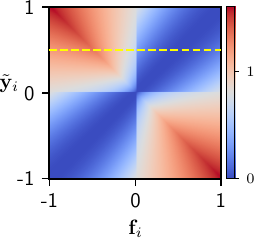}}
    \hfill
    \subfloat[Simplified Focal-L1 loss]{
    \includegraphics[width=0.23\textwidth]{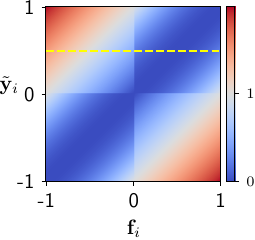}}        
    \newline
    \subfloat[An unexpected local minima at $-1$ ($\bty=0.5$)\label{fig:pitfall_focall1}]{
    \includegraphics[width=0.46\textwidth]{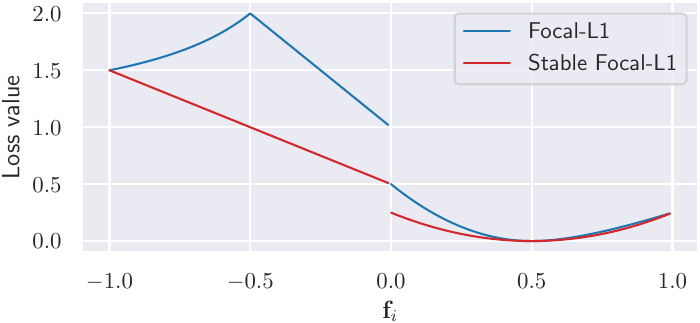}}
    \caption{2D loss surfaces of the Focal-L1 loss and its stable version with $\gamma=1$. Colors represent loss magnitudes. The dashed yellow lines in (a-b) indicate the projections shown in (c).
    }
    \label{fig:focalf1_2d_surface}
\end{figure}

\subsection{Stable Focal-L1 Loss}
\label{sc:focal_l1}

Given a pair of prediction and a soft label $\mathbf{f}, \bty \in [-1,1]^D$, Focal-L1, introduced in~\cite{dang2024singr} for pixel-level regression, is formulated as follows
\begin{equation}
    \mathcal{L}_{\mathrm{FocalL1}} (\bty, \mathbf{f}) = \frac{1}{| \Omega |} \sum_{i \in \Omega} |\bty_i - \mathbf{f}_i| \underbrace{\frac{|\bty_i - \mathbf{f}_i|^{\gamma \mathbbm{I}(\bty_i \mathbf{f}_i \geq 0)}}{\max(|\bty_i|, |\mathbf{f}_i|)} }_{\mathrm{Sample\ weighting}}, \label{eq:focal_l1_loss}
\end{equation}
where $\gamma$ is a positive hyperparameter, and $\mathbb{I}(\cdot)$ is the indicator function. The sample weighting term allows Focal-L1 to prioritize hard pixels over easy ones. 

We observe that the denominator of the weighting term leads to unexpected local minima, as graphically demonstrated in~\cref{fig:focalf1_2d_surface}. Specifically, for any $\bty_i \notin \{-1,1\}$, there are always two minima, one of which is an unexpected local minimum, as shown in~\cref{fig:pitfall_focall1}. Therefore, we here eliminate the denominator from Focal-L1 to form a stable version of Focal-L1 as follows
\begin{equation}
    \mathcal{L}^{S}_{\mathrm{FocalL1}} (\mathbf{f}, \bty) = \frac{1}{| \Omega |} \sum_{i \in \Omega} |\bty_i - \mathbf{f}_i||\bty_i - \mathbf{f}_i|^{\gamma \mathbbm{I}(\bty_i \mathbf{f}_i \geq 0)}. \label{eq:focal_l1_loss}
\end{equation}
In~\cref{prop:stable_focall1}, we show that the Stable Focal-L1 can address the aforementioned issue of Focal-L1~\cite{dang2024singr}. Furthermore, we prove that the Stable Focal-L1 loss is a lower bound of Focal-L1 in~\cref{prop:lowerbound}.
\begin{proposition}[Stable Focal-L1]
\label{prop:stable_focall1}
    Let $\ell : [-1,1] \times [-1,1] \rightarrow \mathbb{R}$ be defined by $\ell(x,y)=|y - x||y - x|^{\gamma \mathbbm{I}(y x \geq 0)}$. Given an arbitrary fixed $y_0\in [-1,1]$, we have that $\ell(x,y_0)$ has only one strictly local and global minimum at $x=y_0$.
    \begin{proof}
        In Supplementary Sec. \textcolor{red}{2}.
    \end{proof}
\end{proposition}
\begin{proposition}[Stable Focal-L1 as a lower bound of Focal-L1]
    \label{prop:lowerbound}
    $\mathcal{L}^{S}_{\mathrm{FocalL1}} (\mathbf{f}, \bty) \leq \mathcal{L}_{\mathrm{FocalL1}} (\mathbf{f}, \bty), \forall \mathbf{f}, \bty \in [-1, 1]^D$.
    \begin{proof}
        In Supplementary Sec. \textcolor{red}{3}.
    \end{proof}
\end{proposition}

\section{Experiments}

\subsection{Experimental Setup}
\paragraph{Datasets} 

We conducted our experiments on five distinct retinal datasets: FIVES, DRIVE, STARE, CHASE-DB1, and HRF, each offering unique characteristics that contribute to a comprehensive evaluation.

FIVES~\cite{jin2022fives} stands out among them for having significantly more samples. The well-structured and sizeable FIVES dataset provides a solid foundation for training and testing, with an official data split of $600$ samples allocated for training and $200$ for testing. In contrast, the other four datasets—DRIVE~\cite{staal2004ridge}, STARE~\cite{hoover2000locating}, CHASE-DB1, and HRF~\cite{budai2013robust}—contain relatively fewer samples, with DRIVE having $60$, STARE $20$, CHASE-DB1 $28$, and HRF $45$ samples, respectively.

Another noteworthy aspect of these datasets is their variation in image resolution. This diversity in image sizes poses both challenges and opportunities for the development and testing of robust image processing algorithms. The image dimensions for FIVES, DRIVE, STARE, CHASE-DB1, and HRF are $2048\times2048$, $584\times565$, $605\times700$, $960\times999$, and $2336\times3504$, respectively. Leveraging this range of resolutions tested our methods' adaptability to different scalabilities and ensured their generalizability across various retinal imaging contexts.

\paragraph{Training and evaluation protocols.} 

As the FIVES dataset is the largest among the available datasets, boasting at least $13$ times more samples than each of the other four datasets, we leveraged data exclusively from FIVES for our training purposes. In contrast, the other datasets were treated as external test sets to ensure an independent evaluation. Specifically, we performed model selection by utilizing $600$ samples from the official FIVES training data split. The remaining $200$ samples, in conjunction with $153$ samples drawn from the other datasets, were set aside for independent testing. This approach allowed us to comprehensively evaluate the generalizability and robustness of our models.

We explored two distinct input settings to better understand the model performance across varying image resolutions: low-resolution (LR) and high-resolution (HR). For the low-resolution setting, the entire RIs were resized to dimensions of $512\times512$ pixels. Conversely, in the high-resolution setting, patches were randomly cropped from the high-resolution RIs. These patches were then resized to a uniform size of $512\times512$ pixels. Hence, the LR and HR settings are referred to as ``full-image'' and ``patch-based'' approaches, respectively.

For the methodology adopted in our experiments, we prioritized the efficient full-image setting (LR). This decision was based on several considerations such as computational efficiency and the ability to capture global image context. By employing this strategy, we aimed to strike a balance between performance and resource utilization.

\paragraph{Baselines}
We compared our method to a diverse set of state-of-the-art references, both general and specific to RIs. They included UNet~\cite{ronneberger2015u}, UNet++~\cite{zhou2019unet++}, CE-NET~\cite{gu2019net}, CTF-Net~\cite{wang2020ctf}, DUNet~\cite{wang2019dual}, FR-UNet~\cite{liu2022full}, IterNet~\cite{li2020iternet}, MAGF-Net~\cite{li2023magf}, Swin-UNet~\cite{cao2022swin}, D2SF~\cite{qiu2023rethinking}, and DA-Net~\cite{wang2022net}. 
Among these, IterNet, FR-UNet, DUNet, CE-Net, CTF-Net, and MAGF-Net were specifically HR-based baselines.
In addition, we incorporated soft-label-based baselines such as label smoothing (LS)~\cite{silva2021using}, boundary LS (BLS)~\cite{wang2023jaccard}, and Geodesic LS (GeoLS)~\cite{vasudeva2023geols}. 

\begin{figure}[t]
    \centering
    \croppdf{figures/SAUNAR_improvements}
    \includegraphics[width=0.47\textwidth]{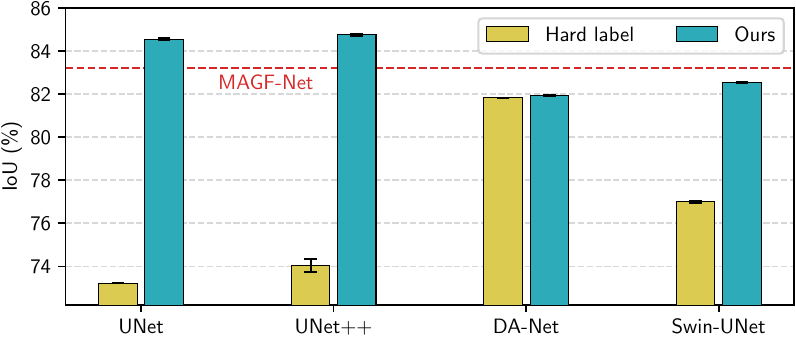}
        \caption{Performance gains of different DL architectures utilizing our approach with LR inputs. The dashed red line indicates the performance of the most competitive HR-based baseline, MAGF-Net~\cite{li2023magf}.}
    \label{fig:perf_improvement}
\end{figure}

\paragraph{Implementation details.}
We conducted our experiments on Nvidia V100 GPUs. Our method and baselines were implemented in Pytorch. 
We applied our method to four LR-based segmentation models: UNet, UNet++, Swin-UNet, and DA-Net.
The soft-label baselines also used the UNet++ network. We ensured that our method and baselines were trained using the same data preprocessing and augmentation pipeline.
During training, we applied data augmentation using random flipping, rotation, color jittering, gamma correction, Gaussian noises, and cutout. Finally, we normalized the images with a mean of $[0.07, 0.15, 0.34]$ and a standard deviation of $[0.2, 0.3, 0.4]$, calculated from the training set of FIVES. 

We used the Adam optimizer to train our method with an initial learning rate of $1e\mathrm{-}4$ and a batch size of $4$. We employed $0$ as the threshold to binarize the SAUNA maps.
While our method, as well as other full-image approaches, took $300$ epochs to train, we spent only $20$ epochs to train patch-based methods due to the enormous number of cropped patches (i.e.\ $800$ patches per RI). 

Each method was re-trained $5$ times with different random seeds. For each random seed, we performed the $5$-fold cross-validation strategy for model selection. The prediction on each test input image was the average of the outputs from $5$ best models in each fold to reduce the effects of random data splitting on our results. 

\paragraph{Evaluation metrics}
For performance assessments, we adopted Dice score, intersection-over-union (IoU), sensitivity (Sens), specificity (Spec), and balanced accuracy (BA; an average of Sens and Spec). We reported image-wise means and standard errors (SE) of test metrics over $5$ runs.

\begin{figure}[t]
    \centering
    \croppdf{figures/SAUNAR_inout_comp_FIVES}
    \includegraphics[width=\linewidth]{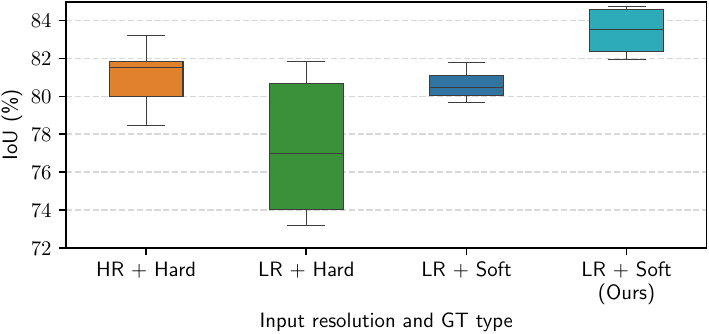}
    \caption{Performance comparison between different groups of methods on FIVES}
    \label{fig:fives_group_comparison}
\end{figure}

\begin{figure}[t]
    \centering
    \croppdf{figures/SAUNAR_inout_comp_STARE}    
    \croppdf{figures/SAUNAR_inout_comp_DRIVE} 
    \croppdf{figures/SAUNAR_inout_comp_CHASEDB1} 
    \croppdf{figures/SAUNAR_inout_comp_HRF} 
    \begin{subfigure}{0.48\linewidth}
        \includegraphics[width=\linewidth]{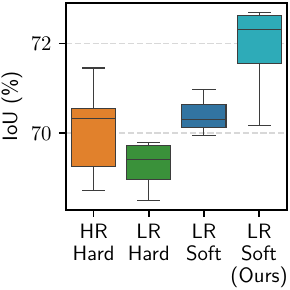}
        \caption{STARE}
    \end{subfigure} \hfill
    \begin{subfigure}{0.48\linewidth}
        \includegraphics[width=\linewidth]{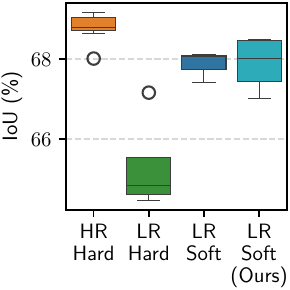}
        \caption{DRIVE}
    \end{subfigure}
    \newline
    \begin{subfigure}{0.48\linewidth}
        \includegraphics[width=\linewidth]{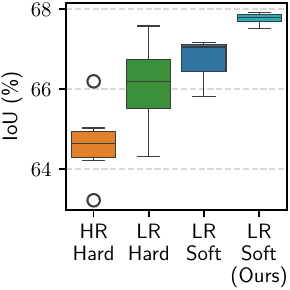}
        \caption{CHASEDB1}
    \end{subfigure} \hfill
    \begin{subfigure}{0.48\linewidth}
        \includegraphics[width=\linewidth]{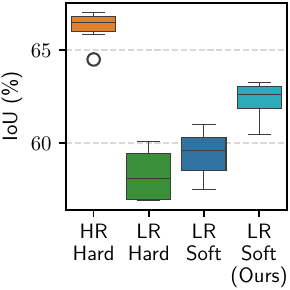}
        \caption{HRF}
    \end{subfigure}
    \caption{Performance comparison between different groups of methods on the four external datasets}
    \label{fig:others_group_comparison}
\end{figure}

\subsection{Results}

\paragraph{FIVES dataset.}
We present the graphical illustrations in~\cref{fig:hr_lr_comparison} and the quantitative results in~\cref{tab:exp_arch_comparisons,fig:perf_improvement,fig:fives_group_comparison}. In general, as shown in~\cref{fig:hr_lr_comparison}, HR-based baselines performed substantially better than most of the LR-based counterparts, albeit with significant throughput trade-offs. The aggregated results in~\cref{fig:fives_group_comparison} indicate that the combination of LR images and hard labels was the less effective. Utilizing soft labels with LR images resulted in improved performance, though it still lagged behind the approach using HR images with hard labels. Notably, our soft-label-based approach with LR outperformed the expensive combination of HR images and hard labels.

Among all the baselines, MAGF-Net~\cite{li2023magf} attained the highest Dice and IoU scores. Our method, utilizing UNet and UNet++, not only surpassed this baseline across all metrics but was also $208$ and $152$ times more computationally efficient, respectively.
Additionally, compared to DA-Net~\cite{wang2022net}, the best LR-based baseline using hard labels, our method with UNet++ achieved substantial improvements of $2.93\%$ in IoU, $1.77\%$ in Dice, and $1.0\%$ in BA. 

The results in~\cref{fig:perf_improvement} demonstrate that the combination of the SAUNA transform and GJML led to performance gains over all four models. Particularly, UNet, UNet++, and Swin-UNet achieved significant improvements of $11.33\%$ $10.7\%$ and $5.52\%$ in IoU, respectively. Furthermore, among the baselines using soft labels, the combination of BLS and JML~\cite{wang2023jaccard} yielded the best performance. When compared to that baseline, our method with UNet++ performed $2.98\%$, $1.75\%$, and $0.49\%$ better in IoU, Dice, and BA, respectively. We visualize the qualitative  results in~\Cref{fig:predictions_viz}.


\begin{table*}[t]
    \centering
    \renewcommand{\arraystretch}{1.1}
    \caption{Performance comparisons between our method (highlighted in \colorbox{LightCyan}{cyan}) and a diverse set of baselines on the FIVES test set. The best results are highlighted in bold. ``IN'' indicates either HR or LR-based approaches. All baseline methods were retrained on our data split for fair comparison.}
    \resizebox{\textwidth}{!}{
    \begin{tabular}{|l|c|c|c|c|c|c|c|c|c|}
\hline
\textbf{Method} &  \textbf{IN} & \textbf{GT} & \textbf{Loss} & \textbf{Imgs/s} & \textbf{IoU} & \textbf{Dice} & \textbf{Sens} & \textbf{Spec} & \textbf{BA} \\
\hline \hline
IterNet~\cite{li2020iternet} & \multirow{8}{*}{\rotatebox[origin=c]{90}{High resolution}} & \multirow{13}{*}{\rotatebox[origin=c]{90}{Hard labels}}
& \multirow{13}{*}{\rotatebox[origin=c]{90}{Dice + BCE}} & 0.67 & 66.90$_{\pm0.53}$ & 78.52$_{\pm0.34}$ & 76.06$_{\pm0.58}$ & 98.90$_{\pm0.01}$ & 87.48$_{\pm0.29}$ \\
FR-UNet~\cite{liu2022full} & & & & 0.34 & 78.44$_{\pm0.29}$ & 87.17$_{\pm0.20}$ & 86.79$_{\pm0.56}$ & 99.14$_{\pm0.04}$ & 92.96$_{\pm0.26}$ \\
DUNet~\cite{wang2019dual} & & & & 0.24 & 80.49$_{\pm0.25}$ & 88.55$_{\pm0.16}$ & 90.02$_{\pm0.39}$ & 99.07$_{\pm0.02}$  & 94.54$_{\pm0.19}$ \\
CE-Net~\cite{gu2019net} & & & & 0.59 & 81.40$_{\pm0.13}$ & 89.14$_{\pm0.08}$ & 90.75$_{\pm0.24}$ & 99.13$_{\pm0.01}$ & 94.94$_{\pm0.12}$ \\
UNet++~\cite{zhou2019unet++} & & & & 1.27 & 81.70$_{\pm0.16}$ & 89.25$_{\pm0.10}$ & 89.38$_{\pm0.26}$ & 99.25$_{\pm0.01}$ & 94.32$_{\pm0.12}$ \\
UNet~\cite{ronneberger2015u} & & & & 3.37 & 81.84$_{\pm0.16}$ & 89.39$_{\pm0.10}$ & 90.94$_{\pm0.26}$ & 99.12$_{\pm0.01}$ & 95.03$_{\pm0.12}$ \\
CTF-Net~\cite{wang2020ctf} & & & & 1.74 & 81.85$_{\pm0.16}$ & 89.51$_{\pm0.10}$ & 90.33$_{\pm0.26}$ & 99.22$_{\pm0.01}$ & 94.78$_{\pm0.12}$ \\
MAGF-Net~\cite{li2023magf} & & & & 0.20 & 83.21$_{\pm0.16}$ & 90.23$_{\pm0.10}$ & 90.71$_{\pm0.26}$ & 99.30$_{\pm0.01}$ & 95.01$_{\pm0.12}$ \\
\cline{1-2} \cline{5-10}
UNet~\cite{ronneberger2015u} &  & & & 38.89 & 73.21$_{\pm0.18}$ & 84.15$_{\pm0.12}$ & 84.08$_{\pm0.23}$ & 98.86$_{\pm0.01}$ & 91.47$_{\pm0.11}$ \\
UNet++~\cite{zhou2019unet++} & & & & 30.45 & 74.05$_{\pm0.31}$ & 84.69$_{\pm0.20}$ & 84.45$_{\pm0.36}$ & 98.92$_{\pm0.03}$ & 91.69$_{\pm0.18}$ \\
Swin-UNet~\cite{cao2022swin} & & & & 49.52 & 77.00$_{\pm0.04}$ & 86.52$_{\pm0.03}$ & 88.12$_{\pm0.10}$ & 98.93$_{\pm0.01}$ & 93.53$_{\pm0.04}$ \\
D2SF~\cite{qiu2023rethinking} & & & & - & 80.68$_{\pm0.20}$ & 89.30$_{\pm0.12}$ & 86.52$_{\pm0.24}$ & \subbest{99.44$_{\pm0.03}$} & 92.98$_{\pm0.12}$ \\
DA-Net~\cite{wang2022net}& & & & 37.31 & 81.82$_{\pm0.05}$ & 89.41$_{\pm0.03}$ & 88.96$_{\pm0.08}$ & 99.34$_{\pm0.01}$ & 94.15$_{\pm0.04}$ \\
\cline{1-1} \cline{3-10}
GeoLS~\cite{vasudeva2023geols,wang2023jaccard} & &  & \multirow{3}{*}{JML} & 30.45 & 79.66$_{\pm0.25}$ & 88.12$_{\pm0.14}$ & \subbest{91.08$_{\pm0.39}$} & 98.91$_{\pm0.07}$ & 94.99$_{\pm0.16}$ \\
LS~\cite{szegedy2016rethinking,wang2023jaccard} & &  &  & 30.45 & 80.47$_{\pm0.27}$ & 88.57$_{\pm0.17}$ & 88.14$_{\pm1.04}$ & 99.27$_{\pm0.09}$ & 93.70$_{\pm0.48}$ \\
BLS~\cite{wang2023jaccard} & & &  & 30.45 & 81.77$_{\pm0.07}$ & 89.43$_{\pm0.04}$ & 90.09$_{\pm0.29}$ & 99.23$_{\pm0.03}$ & 94.66$_{\pm0.13}$ \\
\cline{1-1} \cline{4-10}
\rowcolor{LightCyan}
Ours (DA-Net) &    &    &      & 35.90 & 81.95$_{\pm0.03}$ & 89.52$_{\pm0.02}$ & 88.50$_{\pm0.11}$ & 99.39$_{\pm0.01}$ & 93.95$_{\pm0.05}$ \\
\rowcolor{LightCyan}
Ours (Swin-UNet) &    &    &   GJML +    & 49.52 & 82.52$_{\pm0.03}$ & 89.87$_{\pm0.02}$ & 89.36$_{\pm0.06}$ & 99.39$_{\pm0.01}$ & 94.37$_{\pm0.03}$ \\
   \rowcolor{LightCyan}
     Ours (UNet) &    &    &   SF-L1   & 41.67 & 84.54$_{\pm0.05}$ & 91.04$_{\pm0.05}$ & 90.81$_{\pm0.14}$ & 99.44$_{\pm0.01}$ & 95.13$_{\pm0.06}$ \\
     \rowcolor{LightCyan}
     Ours (UNet++) &  \multirow{-12}{*}{\rotatebox[origin=c]{90}{Low resolution}}  &  \multirow{-7}{*}{\rotatebox[origin=c]{90}{Soft labels}}  &   & 30.45 & \subbest{84.75$_{\pm0.04}$} & \subbest{91.18$_{\pm0.03}$} & 90.85$_{\pm0.11}$ & \subbest{99.46$_{\pm0.01}$} & \subbest{95.15$_{\pm0.05}$} \\
 \hline
    \end{tabular}
    }
    \label{tab:exp_arch_comparisons}
\end{table*}
\paragraph{Generalization to other four datasets.}
In \cref{fig:others_group_comparison}, we present the results of four different method groups, categorized based on their input image resolution (LR or HR) and target type (hard or soft labels). Compared to the baseline group that uses LR images and hard labels, our group demonstrates significantly superior performance on all four datasets. Between the two baseline groups using LR images, those utilizing soft labels show more improvements than those with hard labels. Within the two groups employing soft labels, our group substantially outperformed the baseline group on STARE, CHASEDB1, and HRF.
Moreover, the combination of HR images and hard labels with its computational advance exhibited its strength on DRIVE and HRF. However, our approach generalized substantially better on STARE and CHASEDB1. Interestingly, this expensive setting was the less effective on the CHASEDB1.

In~\cref{tab:external_results}, we present quantitative results of LR-based methods on the four datasets. Generally, all four of our settings consistently outperformed their respective baselines. The most competitive baseline was a soft-label-based method, BLS~\cite{wang2023jaccard}. Compared to this reference, our method with UNet++ results in IoU gains of $1.7\%$, $0.4\%$, $0.5\%$, and $2.3\%$ on STARE, DRIVE, CHASEDB1, and HRF, respectively. 

\paragraph{Ablation study.}
We examined the impact of the signed distance transform and thickness information in~\cref{tab:exp_ablation_sauna}.
The vessel thickness on its own achieved poor segmentation quality, while the signed distance transform could deliver competitive performance. However, when the thickness was used to enrich the signed distance transform, the IoU was improved by $1.88\%$, which is $23.5$ times the standard error.

In~\cref{tab:exp_ablation_losses}, we analyzed the contributions of the GJML and the stable Focal-L1 loss in our method. Our findings indicate that both the stable Focal-L1 loss and GJML individually outperformed the Focal-L1 loss used in~\cite{dang2024singr}. Specifically, the removal of GJML and stable Focal-L1 resulted in IoU drops of $0.32\%$ and $0.39\%$, respectively, which correspond to $8$ and $9.75$ times the standard error.



\begin{figure*}[tbp]
    \centering
    \croppdf{figures/viz/sample0}
    \croppdf{figures/viz/sample1}
    \croppdf{figures/viz/sample2}
    \croppdf{figures/viz/sample_ours}
    \croppdf{figures/viz/sample_magf}
    \croppdf{figures/viz/sample_danet}
    \croppdf{figures/viz/sample_geols}
    \includegraphics[width=0.135\textwidth]{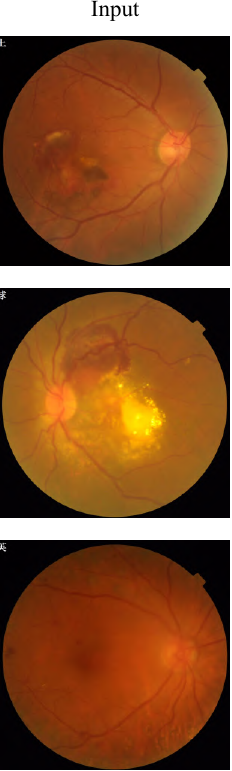} 
    \includegraphics[width=0.135\textwidth]{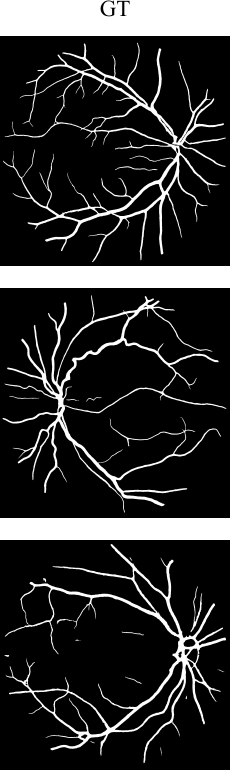} 
    \includegraphics[width=0.135\textwidth]{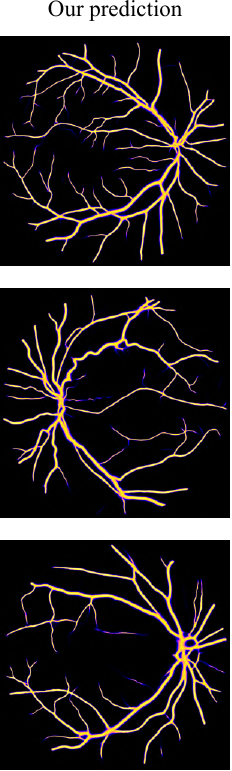} 
    \includegraphics[width=0.135\textwidth]{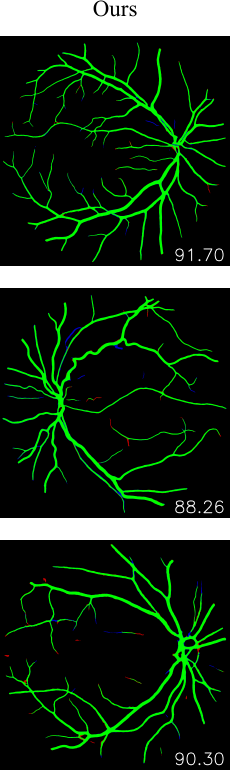} 
    \includegraphics[width=0.135\textwidth]{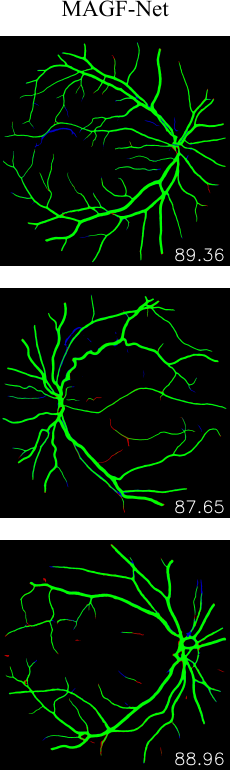} 
    \includegraphics[width=0.135\textwidth]{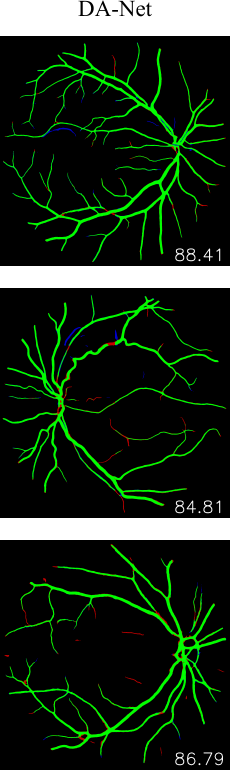} 
    \includegraphics[width=0.135\textwidth]{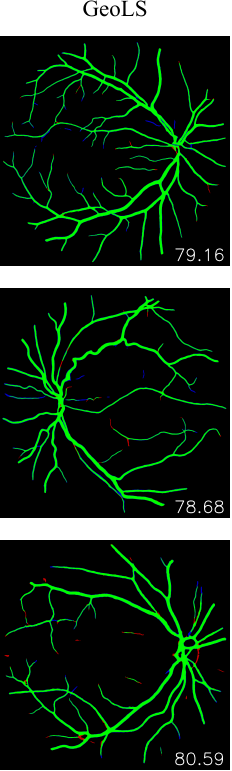} 
    \caption{Visualization of predictions of our method and the baselines on the test sets. HR indicates the method using high-resolution input images. The overlaid values are IoU scores. Green, black, blue, and red pixels indicate true positive, true negative, false positive, and false negative, respectively.}
    \label{fig:predictions_viz}
\end{figure*}

\begin{table}[t]
    \centering
    \renewcommand{\arraystretch}{1.2}
    \caption{Generalization comparisons on the 4 external test sets (IoU means and SEs over 5 runs). Our method is marked in \colorbox{LightCyan}{cyan}. The best results are highlighted in bold.}
    \resizebox{0.48\textwidth}{!}{
    \begin{tabular}{|l|c|c|c|c|c|}
\hline
    \textbf{Method} &   \textbf{GT}  &       \textbf{STARE} &           \textbf{DRIVE} &        \textbf{CHASEDB1} &             \textbf{HRF} \\       
\hline \hline
     UNet~\cite{ronneberger2015u} & \multirow{4}{*}{\rotatebox[origin=c]{90}{Hard labels}} &69.1$_{\pm0.2}$ & 64.5$_{\pm0.2}$ & 65.9$_{\pm0.2}$ & 57.0$_{\pm0.3}$ \\
   UNet++~\cite{zhou2019unet++} & &69.8$_{\pm0.2}$ & 64.7$_{\pm0.4}$ & 66.5$_{\pm0.1}$ & 56.9$_{\pm0.5}$ \\
   DA-Net~\cite{wang2022net} & &69.7$_{\pm0.2}$ & 67.2$_{\pm0.1}$ & 67.6$_{\pm0.1}$ & 60.1$_{\pm0.1}$ \\
Swin-UNet~\cite{cao2022swin} & &68.5$_{\pm0.1}$ & 65.0$_{\pm0.1}$ & 64.3$_{\pm0.1}$ & 59.2$_{\pm0.1}$ \\
\hline
GeoLS~\cite{vasudeva2023geols,wang2023jaccard} & &69.9$_{\pm0.1}$ & 68.1$_{\pm0.1}$ & 65.8$_{\pm0.2}$ & 57.5$_{\pm0.2}$ \\
LS~\cite{szegedy2016rethinking,wang2023jaccard} & &70.3$_{\pm0.3}$ & 67.4$_{\pm0.4}$ & 67.0$_{\pm0.4}$ & 59.6$_{\pm0.3}$ \\
BLS~\cite{wang2023jaccard} & &71.0$_{\pm0.1}$ & 68.1$_{\pm0.1}$ & 67.2$_{\pm0.1}$ & 61.0$_{\pm0.2}$ \\
\rowcolor{LightCyan}
Ours (DA-Net) & & 70.2$_{\pm0.2}$ & 67.6$_{\pm0.0}$ & \subbest{67.9$_{\pm0.1}$} & 60.5$_{\pm0.1}$ \\
\rowcolor{LightCyan}
Ours (Swin-UNet) & & 72.0$_{\pm0.0}$ & 67.0$_{\pm0.0}$ & \subbest{67.9$_{\pm0.0}$} & 62.3$_{\pm0.1}$ \\
\rowcolor{LightCyan}
     Ours (UNet) & & 72.6$_{\pm0.0}$ & \subbest{68.5$_{\pm0.1}$} & 67.5$_{\pm0.0}$ & 62.9$_{\pm0.1}$ \\
     \rowcolor{LightCyan}
   Ours (UNet++) & \multirow{-7}{*}{\rotatebox[origin=c]{90}{Soft labels}} & \subbest{72.7$_{\pm0.1}$} & \subbest{68.5$_{\pm0.1}$} & 67.7$_{\pm0.1}$ & \subbest{63.3$_{\pm0.1}$} \\
\hline
\end{tabular}    
}
    \label{tab:external_results}
\end{table}

\section{Conclusion}
In this work, we have presented a regression-based approach to RV segmentation. We utilized the newly developed SAUNA transform to generate soft labels, motivated by the uncertainty in the annotation process. We leveraged the Jaccard metric loss~\cite{wang2023jaccard} and proved that it is a semi-metric loss on arbitrary hypercubes. In addition, we propose a stable version of the Focal-L1 loss~\cite{dang2024singr}, which 
directly addresses the challenges posed by the unexpected local minima encountered with the original Focal-L1 loss.
Through rigorous experimental evaluation, we showed that our method outperforms existing methods using either LR or HR input images on an in-domain test set (i.e.\ FIVES), and generalizes better compared to LR-based references on external datasets.

\begin{table}[t]
    \centering
    \caption{Ablation study on SAUNA's components and the proposed losses}
    \subfloat[SAUNA\label{tab:exp_ablation_sauna}]{
    \scalebox{0.843}{
        \begin{tabular}{|l|c|}
\hline
\textbf{Setting} & \textbf{IoU} \\
\hline \hline
SAUNA & \subbest{84.75$_{\pm0.04}$} \\
\quad without $\by^T$ & 82.87$_{\pm0.08}$  \\
\quad without $\by^B$ & 72.34$_{\pm3.32}$ \\
\hline
    \end{tabular} 
    }
    }
    \hfill
    \subfloat[Losses\label{tab:exp_ablation_losses}]{
    \scalebox{0.843}{
    \begin{tabular}{|l|c|}
\hline
\textbf{Setting} & \textbf{IoU}  \\
\hline \hline
GJML + SF-L1 & \subbest{84.75$_{\pm0.04}$}  \\
\quad Only SF-L1 & 84.43$_{\pm0.04}$ \\
\quad Only GJML & 84.36$_{\pm0.03}$  \\
\hline
Focal-L1~\cite{dang2024singr} & 84.01$_{\pm0.14}$ \\
\hline
    \end{tabular}  
    }
    }
\end{table}


\newpage
%
%
%
\bibliographystyle{splncs04}
\bibliography{ref}

\begin{thebibliography}{10}
\providecommand{\url}[1]{\texttt{#1}}
\providecommand{\urlprefix}{URL }
\providecommand{\doi}[1]{https://doi.org/#1}

\bibitem{alcover2023soft}
Alcover-Couso, R., Escudero-Vinolo, M., SanMiguel, J.C.: Soft labelling for semantic segmentation: Bringing coherence to label down-sampling. arXiv preprint arXiv:2302.13961  (2023)

\bibitem{bankhead2012fast}
Bankhead, P., Scholfield, C.N., McGeown, J.G., Curtis, T.M.: Fast retinal vessel detection and measurement using wavelets and edge location refinement. PloS one  \textbf{7}(3),  e32435 (2012)

\bibitem{berman2018lovasz}
Berman, M., Triki, A.R., Blaschko, M.B.: The lov{\'a}sz-softmax loss: A tractable surrogate for the optimization of the intersection-over-union measure in neural networks. In: Proceedings of the IEEE conference on computer vision and pattern recognition. pp. 4413--4421 (2018)

\bibitem{bertels2019optimizing}
Bertels, J., Eelbode, T., Berman, M., Vandermeulen, D., Maes, F., Bisschops, R., Blaschko, M.B.: Optimizing the dice score and jaccard index for medical image segmentation: Theory and practice. In: Medical Image Computing and Computer Assisted Intervention--MICCAI 2019: 22nd International Conference, Shenzhen, China, October 13--17, 2019, Proceedings, Part II 22. pp. 92--100. Springer (2019)

\bibitem{budai2013robust}
Budai, A., Bock, R., Maier, A., Hornegger, J., Michelson, G., et~al.: Robust vessel segmentation in fundus images. International journal of biomedical imaging  \textbf{2013} (2013)

\bibitem{cao2022swin}
Cao, H., Wang, Y., Chen, J., Jiang, D., Zhang, X., Tian, Q., Wang, M.: Swin-unet: Unet-like pure transformer for medical image segmentation. In: European conference on computer vision. pp. 205--218. Springer (2022)

\bibitem{chaudhuri1989detection}
Chaudhuri, S., Chatterjee, S., Katz, N., Nelson, M., Goldbaum, M.: Detection of blood vessels in retinal images using two-dimensional matched filters. IEEE Transactions on medical imaging  \textbf{8}(3),  263--269 (1989)

\bibitem{dang2024singr}
Dang, T., Nguyen, H.H., Tiulpin, A.: Singr: Brain tumor segmentation via signed normalized geodesic transform regression. arXiv preprint arXiv:2405.16813  (2024)

\bibitem{fraz2012approach}
Fraz, M.M., Barman, S.A., Remagnino, P., Hoppe, A., Basit, A., Uyyanonvara, B., Rudnicka, A.R., Owen, C.G.: An approach to localize the retinal blood vessels using bit planes and centerline detection. Computer methods and programs in biomedicine  \textbf{108}(2),  600--616 (2012)

\bibitem{fraz2012ensemble}
Fraz, M.M., Remagnino, P., Hoppe, A., Uyyanonvara, B., Rudnicka, A.R., Owen, C.G., Barman, S.A.: An ensemble classification-based approach applied to retinal blood vessel segmentation. IEEE Transactions on Biomedical Engineering  \textbf{59}(9),  2538--2548 (2012)

\bibitem{gu2019net}
Gu, Z., Cheng, J., Fu, H., Zhou, K., Hao, H., Zhao, Y., Zhang, T., Gao, S., Liu, J.: Ce-net: Context encoder network for 2d medical image segmentation. IEEE transactions on medical imaging  \textbf{38}(10),  2281--2292 (2019)

\bibitem{guo2021sa}
Guo, C., Szemenyei, M., Yi, Y., Wang, W., Chen, B., Fan, C.: Sa-unet: Spatial attention u-net for retinal vessel segmentation. In: 2020 25th international conference on pattern recognition (ICPR). pp. 1236--1242. IEEE (2021)

\bibitem{hinton2015distilling}
Hinton, G., Vinyals, O., Dean, J.: Distilling the knowledge in a neural network. arXiv preprint arXiv:1503.02531  (2015)

\bibitem{hoover2000locating}
Hoover, A., Kouznetsova, V., Goldbaum, M.: Locating blood vessels in retinal images by piecewise threshold probing of a matched filter response. IEEE Transactions on Medical imaging  \textbf{19}(3),  203--210 (2000)

\bibitem{jin2022fives}
Jin, K., Huang, X., Zhou, J., Li, Y., Yan, Y., Sun, Y., Zhang, Q., Wang, Y., Ye, J.: Fives: A fundus image dataset for artificial intelligence based vessel segmentation. Scientific Data  \textbf{9}(1), ~475 (2022)

\bibitem{li2023magf}
Li, J., Gao, G., Liu, Y., Yang, L.: Magf-net: A multiscale attention-guided fusion network for retinal vessel segmentation. Measurement  \textbf{206},  112316 (2023)

\bibitem{li2020iternet}
Li, L., Verma, M., Nakashima, Y., Nagahara, H., Kawasaki, R.: Iternet: Retinal image segmentation utilizing structural redundancy in vessel networks. In: Proceedings of the IEEE/CVF winter conference on applications of computer vision. pp. 3656--3665 (2020)

\bibitem{lin2023stimulus}
Lin, J., Huang, X., Zhou, H., Wang, Y., Zhang, Q.: Stimulus-guided adaptive transformer network for retinal blood vessel segmentation in fundus images. Medical Image Analysis  \textbf{89},  102929 (2023)

\bibitem{liu2022full}
Liu, W., Yang, H., Tian, T., Cao, Z., Pan, X., Xu, W., Jin, Y., Gao, F.: Full-resolution network and dual-threshold iteration for retinal vessel and coronary angiograph segmentation. IEEE Journal of Biomedical and Health Informatics  \textbf{26}(9),  4623--4634 (2022)

\bibitem{liu2022combining}
Liu, Z., He, X., Lu, Y.: Combining unet 3+ and transformer for left ventricle segmentation via signed distance and focal loss. Applied Sciences  \textbf{12}(18), ~9208 (2022)

\bibitem{ma2023enhanced}
Ma, J., Wang, C., Liu, Y., Lin, L., Li, G.: Enhanced soft label for semi-supervised semantic segmentation. In: Proceedings of the IEEE/CVF International Conference on Computer Vision. pp. 1185--1195 (2023)

\bibitem{orlando2016discriminatively}
Orlando, J.I., Prokofyeva, E., Blaschko, M.B.: A discriminatively trained fully connected conditional random field model for blood vessel segmentation in fundus images. IEEE transactions on Biomedical Engineering  \textbf{64}(1),  16--27 (2016)

\bibitem{qiu2023rethinking}
Qiu, Z., Hu, Y., Chen, X., Zeng, D., Hu, Q., Liu, J.: Rethinking dual-stream super-resolution semantic learning in medical image segmentation. IEEE Transactions on Pattern Analysis and Machine Intelligence  (2023)

\bibitem{ricci2007retinal}
Ricci, E., Perfetti, R.: Retinal blood vessel segmentation using line operators and support vector classification. IEEE transactions on medical imaging  \textbf{26}(10),  1357--1365 (2007)

\bibitem{ronneberger2015u}
Ronneberger, O., Fischer, P., Brox, T.: U-net: Convolutional networks for biomedical image segmentation. In: Medical Image Computing and Computer-Assisted Intervention--MICCAI 2015: 18th International Conference, Munich, Germany, October 5-9, 2015, Proceedings, Part III 18. pp. 234--241. Springer (2015)

\bibitem{salehi2017tversky}
Salehi, S.S.M., Erdogmus, D., Gholipour, A.: Tversky loss function for image segmentation using 3d fully convolutional deep networks. In: International workshop on machine learning in medical imaging. pp. 379--387. Springer (2017)

\bibitem{silva2021using}
Silva, J.L., Oliveira, A.L.: Using soft labels to model uncertainty in medical image segmentation. arXiv preprint arXiv:2109.12622  (2021)

\bibitem{staal2004ridge}
Staal, J., Abr{\`a}moff, M.D., Niemeijer, M., Viergever, M.A., Van~Ginneken, B.: Ridge-based vessel segmentation in color images of the retina. IEEE transactions on medical imaging  \textbf{23}(4),  501--509 (2004)

\bibitem{sudre2017generalised}
Sudre, C.H., Li, W., Vercauteren, T., Ourselin, S., Jorge~Cardoso, M.: Generalised dice overlap as a deep learning loss function for highly unbalanced segmentations. In: Deep Learning in Medical Image Analysis and Multimodal Learning for Clinical Decision Support: Third International Workshop, DLMIA 2017, and 7th International Workshop, ML-CDS 2017, Held in Conjunction with MICCAI 2017, Qu{\'e}bec City, QC, Canada, September 14, Proceedings 3. pp. 240--248. Springer (2017)

\bibitem{szegedy2016rethinking}
Szegedy, C., Vanhoucke, V., Ioffe, S., Shlens, J., Wojna, Z.: Rethinking the inception architecture for computer vision. In: Proceedings of the IEEE conference on computer vision and pattern recognition. pp. 2818--2826 (2016)

\bibitem{vasudeva2023geols}
Vasudeva, S.A., Dolz, J., Lombaert, H.: Geols: Geodesic label smoothing for image segmentation. In: Medical Imaging with Deep Learning (2023)

\bibitem{wang2019dual}
Wang, B., Qiu, S., He, H.: Dual encoding u-net for retinal vessel segmentation. In: Medical Image Computing and Computer Assisted Intervention--MICCAI 2019: 22nd International Conference, Shenzhen, China, October 13--17, 2019, Proceedings, Part I 22. pp. 84--92. Springer (2019)

\bibitem{wang2022net}
Wang, C., Xu, R., Xu, S., Meng, W., Zhang, X.: Da-net: Dual branch transformer and adaptive strip upsampling for retinal vessels segmentation. In: International Conference on Medical Image Computing and Computer-Assisted Intervention. pp. 528--538. Springer (2022)

\bibitem{wang2021retinal}
Wang, C., Xu, R., Zhang, Y., Xu, S., Zhang, X.: Retinal vessel segmentation via context guide attention net with joint hard sample mining strategy. In: 2021 IEEE 18th International Symposium on Biomedical Imaging (ISBI). pp. 1319--1323. IEEE (2021)

\bibitem{wang2020ctf}
Wang, K., Zhang, X., Huang, S., Wang, Q., Chen, F.: Ctf-net: Retinal vessel segmentation via deep coarse-to-fine supervision network. In: 2020 IEEE 17th International Symposium on Biomedical Imaging (ISBI). pp. 1237--1241. IEEE (2020)

\bibitem{wang2023jaccard}
Wang, Z., Blaschko, M.B.: Jaccard metric losses: Optimizing the jaccard index with soft labels. arXiv preprint arXiv:2302.05666  (2023)

\bibitem{wang2023dice}
Wang, Z., Popordanoska, T., Bertels, J., Lemmens, R., Blaschko, M.B.: Dice semimetric losses: Optimizing the dice score with soft labels. arXiv preprint arXiv:2303.16296  (2023)

\bibitem{xue2020shape}
Xue, Y., Tang, H., Qiao, Z., Gong, G., Yin, Y., Qian, Z., Huang, C., Fan, W., Huang, X.: Shape-aware organ segmentation by predicting signed distance maps. In: Proceedings of the AAAI Conference on Artificial Intelligence. vol.~34, pp. 12565--12572 (2020)

\bibitem{zhou2019unet++}
Zhou, Z., Siddiquee, M.M.R., Tajbakhsh, N., Liang, J.: Unet++: Redesigning skip connections to exploit multiscale features in image segmentation. IEEE transactions on medical imaging  \textbf{39}(6),  1856--1867 (2019)

\end{thebibliography}
\clearpage


\maketitle
\newtheorem{lemma}[theorem]{Lemma}
\renewcommand{\thepage}{S\arabic{page}} 
\renewcommand{\thetable}{S\arabic{table}}  
\renewcommand{\thefigure}{S\arabic{figure}}     

\setcounter{page}{1}
\setcounter{figure}{0}
\setcounter{table}{0}
\setcounter{section}{0}
\setcounter{equation}{0}
\setcounter{proposition}{0}
\section{Proof of~\cref{prop:generalized_JML_metric}}
\label{sc:prop1_proof}
 \begin{proposition}[Jaccard Metric Loss on a hypercube in $\mathbb{R}^D$]
     \label{prop:generalized_JML_metric}     
     $\Delta_{\mathrm{JML}}$ is a semi-metric in $[\alpha, \beta]^{D} \subseteq \mathbb{R}^D$.
Specifically, $\forall \ba, \bb \in [\alpha, \beta]^{D}$, we have
\begin{enumerate}[(i)]
    \item Reflexivity: $\Delta_{\mathrm{JML}}(\ba, \bb) = 0 \Longleftrightarrow \ba \equiv \bb$
    \item Positivity: $\Delta_{\mathrm{JML}}(\ba, \bb) \geq 0$
    \item Symmetry: $\Delta_{\mathrm{JML}}(\ba, \bb) = \Delta_{\mathrm{JML}}(\bb, \ba)$
\end{enumerate}
\begin{proof}    
For any $\mathbf{a}, \mathbf{b} \in [\alpha, \beta]^D$, JML is defined in~\cite{wang2023jaccard} as
 \begin{align}
     \Delta_{\mathrm{JML}}(\mathbf{a},\mathbf{b})= 1 - \frac{\left \| \mathbf{a} + \mathbf{b}\right \|_1 - \left \| \mathbf{a} - \mathbf{b}\right \|_1}{ \left \| \mathbf{a} + \mathbf{b}\right \|_1 + \left \| \mathbf{a} - \mathbf{b}\right \|_1}.
 \end{align}

\paragraph{(i) Reflexivity.}

If $\Delta_{\mathrm{JML}}(\mathbf{a},\mathbf{b}) = 0$, we can derive $\| \mathbf{a} - \mathbf{b} \|_1 = \sum_{i=1}^D |\ba_i - \bb_i| = 0$. Thus, we have $\ba_i = \bb_i, \forall i=1..D$, which is equivalent to $\ba \equiv \bb$. 

If $\ba \equiv \bb$, we obviously have $\Delta_{\mathrm{JML}}(\mathbf{a},\mathbf{b}) = 0$.

\paragraph{(ii) Positivity.} The property is satisfied because we can rewrite $\Delta_{\mathrm{JML}}$ as follows
\begin{align}
    \Delta_{\mathrm{JML}}(\mathbf{a},\mathbf{b}) = \frac{ 2\left \| \mathbf{a} - \mathbf{b}\right \|_1}{ \left \| \mathbf{a} + \mathbf{b}\right \|_1 + \left \| \mathbf{a} - \mathbf{b}\right \|_1} \geq 0, \forall \mathbf{a}, \mathbf{b} \in [\alpha, \beta]^D
\end{align}

\paragraph{(iii) Symmetry.} As $\| \mathbf{a} + \mathbf{b} \|_1 = \| \mathbf{b} + \mathbf{a} \|_1$ and $\| \mathbf{a} - \mathbf{b} \|_1 = \| \mathbf{b} - \mathbf{a} \|_1, \forall \mathbf{a}, \mathbf{b} \in [\alpha, \beta]^D$, we obviously have $\Delta_{\mathrm{JML}}(\ba, \bb) = \Delta_{\mathrm{JML}}(\bb, \ba)$ and this concludes the proof.
\end{proof}
\end{proposition}

\section{Proof of~\cref{prop:stable_focall1}}
\label{sc:stable_focall1_proof}

\begin{lemma}
    Let $\ell : [-1,1] \times [-1,1] \rightarrow \mathbb{R}$ be defined by $\ell(x,y)=|y - x||y - x|^{\gamma \mathbbm{I}(y x \geq 0)}$. For any fixed $y_0\in [0,1]$ (or $[-1, 0]$), the function $\ell(x,y_0)$ does not have any local infimum at $x\in [-1, 0)$ (or $(0, 1]$).
\begin{proof}
        As $\ell(x,y_0)$ is symmetric, without loss of generality, we assume that $y_0 \in [0, 1]$. 
        For simplicity, we denote $\ell(x) = \ell(x, y_0), \forall x \in [-1, 1]$. First, we rewrite $\ell(x)$ as
\begin{align}
    \ell(x)=\left\{\begin{matrix}
| x - y_0 |^{\gamma + 1} & xy_0 \ge 0\\ 
| x - y_0 | & \mathrm{otherwise}
\end{matrix}\right.
\end{align}
$\forall x\in[-1, 0)$, the function $\ell$ becomes a decreasing linear function
\begin{equation}
    \ell(x) = y_0 - x
\end{equation}
Therefore, the only potential local infimum is at $x \rightarrow 0^-$. However, we have that
\begin{align}    
     \lim_{x \rightarrow 0^-} \ell(x) &= y_0 &  \\
    &\geq y_0^{\gamma + 1} & \triangleright\ \mathrm{For}\ \gamma \geq 1\ \mathrm{and}\ y_0 \in [0, 1] \\
    &=\lim_{x \rightarrow 0^+} \ell(x) & \triangleright\ y_0>0
\end{align}
If $y_0 \neq 0$, then $\lim_{x \rightarrow 0^-} \ell(x) > \lim_{x \rightarrow 0^+} \ell(x)$. Thus, $x \rightarrow 0^-$ is not a local infimum. On the other hand, if $y_0 = 0$, then $x=y_0=0 \notin [-1, 0)$. Here, we conclude the proof.

\end{proof}
\end{lemma}

\begin{proposition}[Stable Focal-L1]
    Let $\ell : [-1,1] \times [-1,1] \rightarrow \mathbb{R}$ be defined by $\ell(x,y)=|y - x||y - x|^{\gamma \mathbbm{I}(y x \geq 0)}$. Given an arbitrary fixed $y_0\in [-1,1]$, we have that $\ell(x,y_0)$ has only one strictly local and global minimum at $x=y_0$.    
\end{proposition}
\begin{proof}    
Let $\ell : [-1,1] \times [-1,1] \rightarrow \mathbb{R}$ be defined by $\ell(x,y)=|y - x||y - x|^{\gamma \mathbbm{I}(y x \geq 0)}$. Consider an arbitrary fixed $y_0\in [-1,1]$. As $\ell(x,y_0)$ is symmetric, without loss of generality, we assume that $y_0 \in [0, 1]$. For simplicity, we denote $\ell(x) = \ell(x, y_0), \forall x \in [-1, 1]$. First, we rewrite $\ell(x)$ as
\begin{align}
    \ell(x)=\left\{\begin{matrix}
| x - y_0 |^{\gamma + 1} & xy_0 \ge 0\\ 
| x - y_0 | & \mathrm{otherwise}
\end{matrix}\right.
\end{align}

\paragraph{(i) $y_0 \in (0, 1]$:}

$\forall x \in [0, 1]$, we have that
\begin{align}
    \ell(x) = |x - y_0|^{\gamma + 1}
    \label{eq:l_pos}
\end{align}
One can observe that 
\begin{equation}
    \ell(x) > \ell(y_0) = 0, \forall x \in [0,1] \backslash \{y_0\}
    \label{eq:case_1}
\end{equation}

Consider an arbitrary $x \in [-1, 0)$,
$\ell$ then becomes a decreasing linear function, that is 
\begin{align}
    \ell(x) = y_0 - x
    \label{eq:l_neg}
\end{align}
Then, we have the following derivations: $\forall x \in [-1, 0),$
\begin{align}
    \ell(x) &\geq \inf_{x \in[-1,0)} \ell(x) &  \label{eq:abc} \\
     &= \lim_{x \rightarrow 0^-} \ell(x) & \\
    &= y_0 & \triangleright\ \text{For \eqref{eq:l_neg}} \\
    &>y_0^{\gamma + 1} & \triangleright\ \mathrm{For}\ \gamma \geq 1\ \mathrm{and}\ y_0 \in (0, 1] \\
    &= \ell(0) & \triangleright\ \text{For \eqref{eq:l_pos} and $y_0 > 0$} \\
    &> \ell(y_0) = 0 & \triangleright\ \text{For \eqref{eq:case_1}} \label{eq:case_2}
\end{align}


From \eqref{eq:case_1} and \eqref{eq:case_2}, we can infer that 
\begin{align}
    \ell(x) > \ell(y_0), \forall x \in [-1,1] \backslash \{ y_0\}.    
\end{align}
In other words, $x=y_0$ is the only strictly global minimum of $\ell$ in $[-1, 1]$.

\paragraph{(ii) $y_0=0$:} We have that
\begin{align}
    \ell(x) = |x|^{\gamma + 1}
\end{align}
Similarly, one can observe that 
\begin{align}
    \ell(0) < \ell(x), \forall x \in [-1, 1] \backslash \{0\},
\end{align}
which implies that $x=0$ is the only strictly global minimum in $[-1,1]$.

From (i) and (ii), we conclude that $x = y_0$ is the only strictly global minimum of $\ell(x)$ in $[-1,1]$. 

Furthermore, $\ell(x)$ is a convex function in $[0, 1]$ as its second derivative is non-negative in this domain, that is
\begin{align}
    \frac{\partial^2 }{\partial x^2}\ell(x) = &2 (\gamma + 1) \delta (x - y_0) | x - y_0 | ^\gamma \\
            &+ \gamma (\gamma + 1) (x - y_0)^2 |x -y_0|^{\gamma - 3} \geq 0, \forall x\in [0,1]
\end{align}
where $\delta$ is the Dirac Delta function. Thus, $\ell$ has at most one local minimum in $[0,1]$, which is $x=y_0$. Together with Lemma 2.1, we conclude that the function $\ell(x)$ has only one strictly local and global minimum at $x=y_0$ in $[-1,1]$.

\end{proof}

\section{Proof of~\cref{prop:lowerbound}}
\label{sc:lowerbound_proof}

\begin{proposition}[Stable Focal-L1 as a lower bound of Focal-L1]
    \label{prop:lowerbound}
    $\mathcal{L}^{S}_{\mathrm{FocalL1}} (\mathbf{f}, \bty) \leq \mathcal{L}_{\mathrm{FocalL1}} (\mathbf{f}, \bty), \forall \mathbf{f}, \bty \in [-1, 1]^D$.
\begin{proof}
    We need to prove that $\mathcal{L}^{S}_{\mathrm{FocalL1}} (\mathbf{f}, \bty) \leq \mathcal{L}_{\mathrm{FocalL1}} (\mathbf{f}, \bty), \forall \mathbf{f}, \bty \in [-1, 1]^D$.

    We denote that 
    \begin{align}
    \ell^S(\bty_i, \mathbf{f}_i) &= |\bty_i - \mathbf{f}_i||\bty_i - \mathbf{f}_i|^{\gamma \mathbbm{I}(\bty_i \mathbf{f}_i \geq 0)}, \\
        \ell(\bty_i, \mathbf{f}_i) &= |\bty_i - \mathbf{f}_i| \frac{|\bty_i - \mathbf{f}_i|^{\gamma \mathbbm{I}(\bty_i \mathbf{f}_i \geq 0)}}{\max(|\bty_i|, |\mathbf{f}_i|)}.
    \end{align}
Then, the two losses become
    \begin{align}
    \mathcal{L}^{S}_{\mathrm{FocalL1}} (\bty, \mathbf{f}) &= \frac{1}{| \Omega |} \sum_{i \in \Omega} \ell^S(\bty_i, \mathbf{f}_i), \\
\mathcal{L}_{\mathrm{FocalL1}} (\bty, \mathbf{f}) &= \frac{1}{| \Omega |} \sum_{i \in \Omega} \ell(\bty_i, \mathbf{f}_i).
    \end{align}

    Because $\max(|\bty_i|, |\mathbf{f}_i|) \leq 1, \forall \bty_i, \mathbf{f}_i \in [-1, 1]$, we straightforwardly derive that $\ell^S(\bty_i, \mathbf{f}_i) \leq \ell(\bty_i, \mathbf{f}_i), \forall i \in \Omega$. Thus, we can conclude the proof.
\end{proof}
\end{proposition}

\end{document}